\documentclass[lettersize,journal]{IEEEtran}
\usepackage{amsmath,amsfonts}
\usepackage{amssymb,amsthm}
\usepackage{algorithmic}
\usepackage{algorithm}
\usepackage{array}
\usepackage[caption=false,font=normalsize,labelfont=sf,textfont=sf]{subfig}
\usepackage{textcomp}
\usepackage{stfloats}
\usepackage{url}
\usepackage{textcomp}
\usepackage{verbatim}
\usepackage{booktabs}
\usepackage{multirow}
\usepackage{cite}
\usepackage{textcomp}
\usepackage{marvosym}
\usepackage{makecell}
\usepackage{bm}
\usepackage[numbers,sort&compress]{natbib}
\usepackage{graphicx}
\usepackage{geometry} 
\theoremstyle{definition}
\newtheorem{proposition}{Proposition}

\newtheorem{definition}{Definition}
\newtheorem{remark}{Remark}
\newtheorem{example}{Example}
\newtheorem{corollary}{Corollary}

\begin{document}

\title{DRoPE: Directional Rotary Position Embedding \\for Efficient Agent Interaction Modeling}

\author{
    Jianbo Zhao$^{1,2}$, Taiyu Ban$^{1}$, Zhihao Liu$^{2}$, Hangning Zhou$^{2,\dagger}$\textsuperscript{\Letter}, Xiyang Wang$^{2}$, Qibin Zhou$^{2}$, Hailong~Qin$^{3}$, Mu~Yang$^{2}$, Lei~Liu$^{1}$\textsuperscript{\Letter}, Bin Li$^{1}$
    \thanks{$^\dagger$ Project Leader: Hangning Zhou (hangning.zhou@mach-drive.com).}
    \thanks{\textsuperscript{\Letter} Corresponding authors: Lei Liu (liulei13@ustc.edu.cn), Hangning Zhou (hangning.zhou@mach-drive.com).}
    \thanks{$^1$ Jianbo Zhao (zjb123@mail.ustc.edu.cn), Taiyu Ban, Lei Liu, and Bin Li are with the University of Science and Technology of China, 96 Jinzhai Rd, Hefei 230026, China.}
    \thanks{$^2$ Jianbo Zhao, Zhihao Liu, Hangning Zhou, Xiyang Wang, Qibin Zhou, and Mu Yang are with Mach Drive, Beijing, China.}
    \thanks{$^3$ Hailong Qin is with Temasek Laboratories, National University of Singapore, Singapore.}
}

\maketitle

\begin{abstract}
Accurate and efficient modeling of agent interactions is essential for trajectory generation, the core of autonomous driving systems. Existing methods, scene-centric, agent-centric, and query-centric frameworks, each present distinct advantages and drawbacks, creating an impossible triangle among accuracy, computational time, and memory efficiency. To break this limitation, we propose Directional Rotary Position Embedding (DRoPE), a novel adaptation of Rotary Position Embedding (RoPE), originally developed in natural language processing. Unlike traditional relative position embedding (RPE), which introduces significant space complexity, RoPE efficiently encodes relative positions without explicitly increasing complexity but faces inherent limitations in handling angular information due to periodicity. DRoPE overcomes this limitation by introducing a uniform identity scalar into RoPE's 2D rotary transformation, aligning rotation angles with realistic agent headings to naturally encode relative angular information. We theoretically analyze DRoPE's correctness and efficiency, demonstrating its capability to simultaneously optimize trajectory generation accuracy, time complexity, and space complexity. Empirical evaluations compared with various state-of-the-art trajectory generation models, confirm DRoPE's good performance and significantly reduced space complexity, indicating both theoretical soundness and practical effectiveness. The video documentation is available at \url{https://drope-traj.github.io/}.
\end{abstract}

\begin{IEEEkeywords}
Deep learning methods, Trajectory generation, Autonomous driving, Relative position embedding.
\end{IEEEkeywords}

\section{Introduction}

\IEEEPARstart{M}{odeling} agent interactions is crucial for trajectory generation, often regarded as the ``brain" of deep learning-based autonomous driving (AD) \cite{jia2025drivetransformer, 10629039, jia2024bench2drive}. Agent interactions typically involve agents' relative spatial positions and velocities, alongside their inherent features. Current agent interaction modeling methods primarily fall into three frameworks: scene-centric, agent-centric, and query-centric \cite{shi2025motion}.

Among these methods, the scene-centric framework is known to exhibit inferior performance due to its use of absolute positions as direct inputs, impairing its ability to effectively capture spatial relationships between agents. Conversely, agent-centric and query-centric frameworks leverage relative position and yield better accuracy. However, agent-centric methods suffer from high time complexity \cite{zhang2024trafficbots}, while query-centric methods experience increased space complexity \cite{zhou2023query}. As a result, these three approaches form an ``impossible triangle", where no single method can simultaneously optimize accuracy, time complexity, and space complexity, as illustrated in Fig. \ref{fig:impTri}.

Specifically, agent-centric approaches designate a focal agent as the origin and transform other agents' coordinates accordingly. Although intuitive for modeling relative positions, this requires repeated coordinate transformations, training, and inference for each agent, leading to an \(N\)-fold increase in computational time complexity, where \(N\) denotes the number of agents. In contrast, query-centric methods typically utilize Relative Position Embeddings (RPE) \cite{zhou2023query} to encode inter-agent relative positions, allowing simultaneous trajectory inference for all agents. However, RPE substantially increases space complexity, by a factor of \(N\), due to the explicit representation of relative positions among agents (see Fig. \ref{fig:rpe}(a)).

\begin{figure}[!t]
    \centering
    \includegraphics[width=0.3\textwidth]{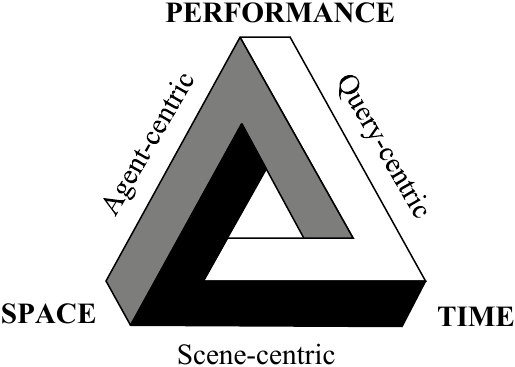}
    \caption{The impossible triangular of current trajectory generation methods.}
    \label{fig:impTri}
\end{figure}

Recently, a novel RPE technique from the natural language processing domain, named Rotary Position Embedding (RoPE) \cite{su2024roformer}, proposes an efficient global position-driven method for embedding relative positions. RoPE embeds global token positions into query-key (QK) vectors using consecutive 2D rotary transformations, naturally encoding relative positions into attention weights through rotary transformations of vector dot products. Since RoPE avoids explicitly representing relative positions between token pairs, it maintains the space and computational complexity of the transformer \cite{vaswani2017attention} model (see Fig. \ref{fig:rpe}(b)). RoPE thus opens a promising avenue for breaking the impossible angle encountered by existing trajectory generation methods.

\begin{figure}[!t]
    \centering
    \includegraphics[width=1.\linewidth]{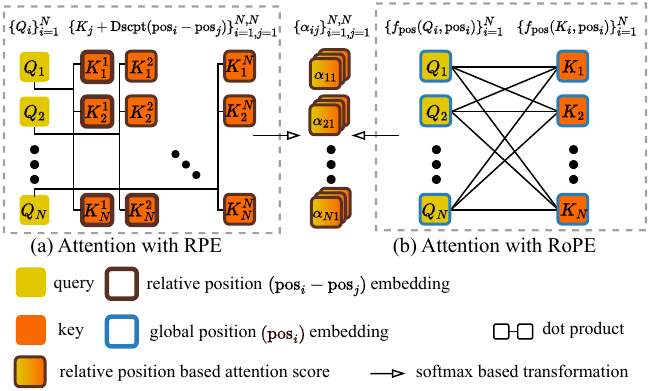}
    \caption{RPE \textit{v.s.} RoPE in terms of space complexity.}
    \label{fig:rpe}
\end{figure}

However, directly applying RoPE to trajectory generation \cite{xu2023context, zhou2023qcnext} is infeasible because it cannot naturally represent agent headings, owing to the inherent periodicity of angular information. Unlike RPE, which directly handles modular transformations of relative angles, RoPE can only operates on global angles, rendering relative angles implicit and inaccessible. Thus, RoPE struggles with periodic angular relations essential in trajectory prediction since it fails to address modular transformations (see Fig.~\ref{fig:rope_angle}).

To address this gap, we propose Directional Rotary Position Embedding (DRoPE), a novel adaptation of RoPE designed specifically for periodic angle modeling. DRoPE introduces a uniform identity scalar into the 2D rotary transformation, effectively aligning rotation angles with realistic agent headings. This modification creates a consistent mapping from real-number fields to periodic angular domains, enabling DRoPE to handle angles naturally within the rotary embedding framework.
By integrating DRoPE with RoPE, we naturally embed both relative positions and headings of agents without significantly increasing computational or space complexity.

We provide thorough theoretical analyses demonstrating DRoPE's correctness and efficiency in space complexity, offering strong theoretical justification for its practical adoption. Furthermore, we empirically validate DRoPE by comparing it with various state-of-the-art models. Our experiments demonstrate that DRoPE significantly reduces space complexity while simultaneously maintaining good prediction performance.
Our contributes are summarizes as below.
 \begin{enumerate} 
 \item We introduce DRoPE, which first adapt rotary position embedding specifically tailored for trajectory generation, achieving high performance, low time complexity, and low space complexity simultaneously. 
 \item We present thorough theoretical analysis of RoPE's limitations in handling angles and rigorously demonstrate the effectiveness of DRoPE in addressing this issue. 
 \item We propose two practical DRoPE-RoPE architectures designed explicitly for agent interaction modeling, accommodating diverse practical intentions. \end{enumerate}

\begin{figure}[!t]
    \centering
    \includegraphics[width=1.\linewidth]{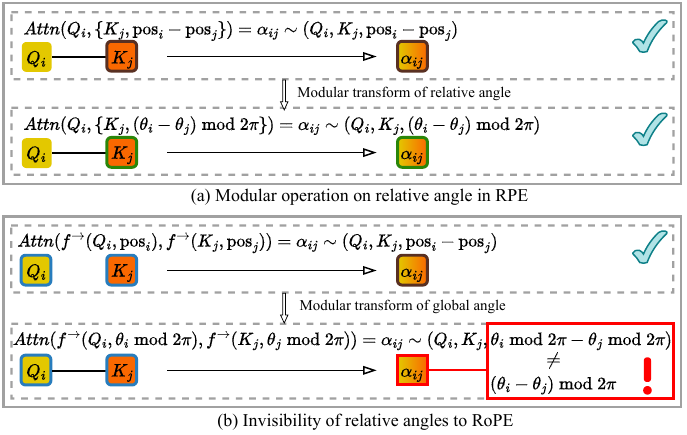}
    \caption{The infeasibility of RoPE in handling the periodicity of angles.}
    \label{fig:rope_angle}
\end{figure}

\section{Related Work}

Modeling agent interactions in autonomous driving involves distinct methodologies, each with inherent strengths and trade-offs. Current approaches primarily fall into three frameworks: scene-centric, agent-centric, and query-centric.

\subsection{Scene-centric approaches}

Scene-centric methods \cite{chen2023categorical, hu2024solving} prioritize computational efficiency by representing all scene elements, including map features and agents, within a unified, fixed coordinate system, commonly centered on the ego vehicle. In these models, a scene encoder maps spatial and structural characteristics into latent embeddings, from which a trajectory decoder subsequently predicts the target agent's future motions. However, the representation of all agents in a single coordinate frame results in data distribution imbalances. Specifically, agents nearer to the ego vehicle are disproportionately represented, while predictive accuracy deteriorates for distant agents. Consequently, while these methods are computationally efficient, their accuracy often remains suboptimal \cite{zhang2024trafficbots}.

\subsection{Agent-centric approaches}

Agent-centric models \cite{feng2023macformer, 10803038} have been proposed to overcome the limitations of scene-centric frameworks by normalizing the coordinate system around individual agents. For a scenario involving \(N\) agents, this approach effectively increases the size and diversity of the training dataset by a factor of \(N\). While this normalization substantially improves predictive accuracy, it significantly raises computational complexity since each agent requires separate coordinate transformations and inference steps. Thus, agent-centric methods incur computational costs that scale linearly with \(N\), limiting their practicality in real-time applications.

\subsection{Query-centric approaches}

The query-centric paradigm, proposed by QCNet \cite{zhou2023query} and extensively adopted thereafter \cite{wu2024smart, zhou2023qcnext, zhou2024behaviorgpt}, attempts to balance computational efficiency with accurate spatial modeling. Unlike the agent-centric framework, query-centric methods decouple shape encoding from spatial relationship encoding. Scene elements' shapes are encoded separately, whereas relative spatial relationships are subsequently captured through multi-head attention, incorporating explicit relative position encodings into key-value vectors.
For a scene of $N$ agents, this results in an \(N^2\)-sized relative position matrix. Each relative position is encoded and embedded into key and value vectors, increasing the space complexity from \(O(N)\) to \(O(N^2)\) and leading to considerable memory overhead. Due to this high memory requirement, existing query-centric methods  typically constrain their attention mechanisms to local neighborhoods, trading off comprehensive spatial interactions for reduced memory usage \cite{wu2024smart, zhang2025closed}.

In this paper, we introduce Directional Rotary Positional Embeddings (DRoPE) to develop a query-centric method that maintains a spatial complexity of \(O(N)\) while preserving the same time complexity. Our method achieves competitive performance compared to approaches with \(O(N^2)\) spatial complexity.

\section{Directional RoPE and Theoretical Analysis}

This section introduces the proposed Directional RoPE mechanism and the theoretical analysis of its correctness.
We begin with the traditional RPE mechanism and its bad space complexity $O(N^2)$ in parallel computation.
Subsequently, we introduce RoPE and show its effective space complexity $O(N)$ surpassing RPE.
Then, we analyze the infeasibility of RoPE in representing angle distances, used to model heading directions of a motion.
Lastly, we introduce our directional RoPE that addresses this misalignment and demonstrates its correctness.

\subsection{Relative positional embedding}

To set the stage, we define symbols useful for the theoretical analysis.
Let \(E=\{E_i\}_{i=0}^{N-1}\) be the set of investigated tokens\footnote{Token $E_i$ and token $i$ are used interchangeably in the paper.} where $E_i\in \mathbb{R}^{d_e}$.
Suppose that there are \( H \) attention heads for each token, which are query $\{Q_{i}^h\}_{{i=1},{h=1}}^{N,H}$ and key $\{K_{i}^h\}_{{i=1},{h=1}}^{N,H}$ vectors where $Q_i^h,K_i^h\in \mathbb{R}^{d_k}$, and value vectors $\{V_{i}^h\}_{{i=1},{h=1}}^{N,H}$ where \(V_{i}^h\in \mathbb{R}^{d_v}\).
In the rest contents of this section, we refer to a QKV vector omitting its superscript, like $Q_i$ for token $E_i$ instead of $Q_i^h$ because $h$ is not relevant for analysis.

Now we delve into the calculation of output for each token. For token $i$, it first conducts attention with all tokens by:
\begin{equation}
\label{eq:attention_score_base}
\alpha_{ij} = \mathrm{softmax} \left( \frac{\langle Q_i, K_j \rangle}{\sqrt{d_k}} \right),\,\, j=1,\cdots,N,
\end{equation}
where $\langle A,B\rangle$ represents dot product $A^T B$.
Then the output of token $i$ is derived by an attention-based weighted sum:
\begin{equation}
\bar{O}_i = \sum_{j=1}^N \beta_{ij} = \sum_{j=1}^{N} \alpha_{ij} V_j.
\label{eq:mult_attention_outcome}
\end{equation}
For multiple heads, this process is repeated, and the outcomes are concentrated and injected into a vector $O_i$. We denote these operations as $\text{MHSA}(E)$.

In the parallel calculation setting where all numerical operations are processed simultaneously, the space used to store the inputs of the attention process is given by the following result.

\begin{proposition}
\label{prop:mha_space}
The space complexity for multi-head attention inputs under parallel computation is \(\mathcal{O}(NH(2d_k+d_v))\).
\end{proposition}

\begin{proof}
Each token $E_i$ has query, key, and value vectors per head, where $Q_i^h, K_i^h \in \mathbb{R}^{d_k}$ and $V_i^h \in \mathbb{R}^{d_v}$. The space required for a single head per token is $(2d_k + d_v)$. With $H$ heads and $N$ tokens, the total space is $N H (2d_k + d_v)$. Thus, the space complexity is $\mathcal{O}(NH(2d_k + d_v))$.
\end{proof}

Next, we focus on the well-known relative positional encoding (RPE) and its influence on space complexity.

\begin{definition}[RPE]
\label{def:RPE}
Given $\text{pos}_i\in \mathbb{R}^{p}$ denoting the position of token $i$, RPE calculates the attention weights by:
\begin{gather}
    \alpha_{ij} = \mathrm{softmax} \left(
  \frac{\langle Q_i, K_{ij} \rangle}{\sqrt{d_k}}
\right), \label{eq:rpe_attention}\\
\text{where } K_{ij} = K_j + \text{Dscpt}_k(\text{pos}_i-\text{pos}_j). \label{eq:rpe_k_embedding}
\end{gather}
Here, $\text{Dscpt}_k: \mathbb{R}^p \rightarrow \mathbb{R}^{d_k}$ represents a learnable encoder.
Then the outcomes of $E_i$ are calculated by:
\begin{gather}
    \bar{O}_i = \sum_{j=1}^N \beta_{ij} = \sum_{j=1}^{N} \alpha_{ij} V_{ij},\label{eq:rpe_weighted_sum}\\
    \text{where }V_{ij} = V_j + \text{Dscpt}_v(\text{pos}_i-\text{pos}_j). \label{eq:rpe_v_embedding}
\end{gather}
Here, $\text{Dscpt}_v:\mathbb{R}^p \rightarrow \mathbb{R}^{d_v}$ is a learnable encoder.

\end{definition}

The essential characteristic of RPE is to encode the relative position between tokens in the attention weights, whose general form is defined below.

\begin{definition}[General RPE functions]
    A function that embeds the relative position between tokens is defined by:
    \begin{equation}
    \beta_{ij} = f(Q_i,K_j,V_j,\text{pos}_i - \text{pos}_j),
    \label{eq:general_RPE_functions}
\end{equation}
where the attention-weighted value $\beta_{ij}$ of tokens $i$ and $j$ are dependent and only dependent on their the relative position $\text{pos}_i - \text{pos}_j$ except for the related QKV vectors.
\label{def:general_RPE_functions}
\end{definition}

\begin{remark}
\label{remark:rpe_relative_position}
Note that when the position refers to an angle $\theta_i$ where $\theta_i + 2\pi = \theta_i$.
The relative position is calculated by:
\begin{equation}
    (\theta_i - \theta_j) \bmod 2\pi,
\end{equation}
to maintain the alignment with the angle definition.
\end{remark}

\begin{proposition}
\label{prop:mha_rpe_space}
The space complexity for RPE (by Definition \ref{def:RPE}) inputs  in parallel computation is \(\mathcal{O}(N^2H(d_k+d_v))\).
\end{proposition}

\begin{proof}
The space required for a single head per token is $(2d_k + d_v)$. 
For the intermediate variables $K_{ij}\in \mathbb{R}^{d_k},V_{ij}\in \mathbb{R}^{d_v}$ used for calculating attention weights in Equation \eqref{eq:rpe_attention} and for calculating attention-weighted sum in Equation \eqref{eq:rpe_weighted_sum}, there are $N^2$ individual such terms for all $i,j\in\{1,2,\cdots,N\}$.
Given $H$ heads, these variables requires $O(N^2H(d_k+d_v))$ space for storage.
Combining with $O(NH(2d_k+d_v))$ space complexity used for QKV vectors, the space complexity for RPE inputs is $O(N^2H(d_k+d_v))$.
Hence, we complete the proof.
\end{proof}

Compared to the space complexity for multi-head attention input, which is $O(NH(2d_k+d_v))$ as shown by Proposition \ref{prop:mha_space}, we observe that RPE significantly increases the space complexity by $N$ times. In the context of motion prediction, $N$ represents the number of agents is typically large in complex realistic traffic scenarios. The increased space complexity by RPE makes it unaffordable to make attention with all agent tokens in reasonable GPU resources. Typically, researchers set $N$, the number of agents visible to the model, to a relatively small value to balance model performance with GPU resources.
Currently, the full modeling of interactions between all agents is considered impossible for RPE-based methods.

\subsection{Rotary position embedding}
Compared to RPE, rotary position embedding (RoPE) is an alternative relative position encoding approach that maintains the same space complexity of multi-attention inputs, which however faces problems when dealing with angles.
We delve into these aspects subsequently.
\begin{definition}[RoPE]
\label{def:RoPE}
Suppose that the QK vectors are in even dimensions $2d_k$.
Given $\text{pos}_i\in \mathbb{R}^{p}$ denoting the position of token $i$, RoPE calculates the attention weights by:
\begin{gather}
    \alpha_{ij} = \mathrm{softmax} \left(
  \frac{ \langle \hat{Q}_i, \hat{K}_j \rangle}{\sqrt{d_k}}
\right), \label{eq:RoPE_attention}\\
\text{where } \hat{Q}_i = f^{\rightarrow}(Q_i, \text{pos}_i),\,\,
\hat{K}_j = f^{\rightarrow}(K_j, \text{pos}_j).
\label{eq:absolute_pos_emb}
\end{gather}
The function $f^{\rightarrow}$ embeds the \textit{absolute} position of tokens in their QK vectors, formally defined by:
\begin{gather}
    f^{\rightarrow}(X,m) = \text{BlockDiag} \big( R(m\theta_0), \dots, R(m\theta_{d_k -1}) \big) X, \label{eq:block_diag}\\
    \text{where } R(m\theta_l) = \begin{bmatrix} \cos(m\theta_l) & -\sin(m\theta_l) \\ \sin(m\theta_l) & \cos(m\theta_l) \end{bmatrix}, \label{eq:R}\\
    \text{and } \theta_l = 10000^{-l/d_k},  \,  l=0,1,\cdots,d_k-1. \label{eq:theta_l}
\end{gather}
Here, \( \text{BlockDiag}(\cdot) \) represents the block diagonal matrix that applies each \( R(m\theta_i) \) to the corresponding 2D vector pair in \( X \).
Then the outcomes of token $i$ are calculated by:
\begin{gather}
    \bar{O}_i = \sum_{j=1}^N \beta_{ij} = \sum_{j=1}^{N} \alpha_{ij} V_{j}.
    \label{eq:RoPE_weighted_sum}
\end{gather}
The weighted sum process of RoPE is exactly the same as that of the multi-attention process as presented in Equation \eqref{eq:mult_attention_outcome}.
\end{definition}

We now show that RoPE satisfies the general characteristic of RPE functions illustrated in Definition \ref{def:general_RPE_functions}.

\begin{corollary}
For $Q_i,K_j\in \mathbb{R}^{2d_k}$ of tokens $i,j$, $\hat{Q}_i$ and $\hat{K}_j$ defined by Equation \eqref{eq:absolute_pos_emb}, their dot product follows:
\begin{equation}
    \langle\hat{Q}_i, \hat{K}_j\rangle = f(Q_i,K_j,\text{pos}_i-\text{pos}_j),
\end{equation}
which is a function only dependent on $Q_i,K_j$ and relative position $\text{pos}_i-\text{pos}_j$ between tokens $i$ and $j$.
\label{coro:RoPE_valid}
\end{corollary}
\begin{proof}
We denote the block diagonal term in Equation \eqref{eq:block_diag} as $B(m)$ for the absolute position $m$, and refer to $\text{pos}_i$ as $m_i$. Then we derive the formulation of $\langle\bar{Q}_i, \bar{K}_j\rangle$:
    \begin{equation}
    \notag
    \begin{aligned}
        \langle\hat{Q}_i, \hat{K}_j\rangle &= \left(B(m_i)Q_i\right)^TB(m_j)K_j\\
        &= Q_i^T B(m_i)^T B(m_j)  K_j \\
        &= Q_i^T  \text{BlockDiag} \big( R(m_i\theta_0)^T, \dots, R(m_i\theta_{d_k -1})^T \big)\\
        &\quad\quad\,\,\,\, \text{BlockDiag} \big( R(m_j\theta_0), \dots, R(m_j\theta_{d_k -1}) \big) K_j\\
        &= Q_i^T \text{BlockDiag}\left(\{R(-m_i\theta_l)R(m_j\theta_l)\}\right) K_j\\
        &=Q_i^T \text{BlockDiag}\left(\{R(-(m_i-m_j)\theta_l)\}\right)K_j
    \end{aligned}
\end{equation}
where $R(\theta)$ is the 2D rotary matrix w.r.t. $\theta$ defined by Equation \eqref{eq:R}, and $\theta_l$ is defined by Equation \eqref{eq:theta_l}.
In this derivation, the last two equalities hold due to the properties of the rotary matrix $R^T(\theta)=R(-\theta)$ and $R(\theta_1)R(\theta_2) = R(\theta_1+\theta_2)$.

We observe that the formulation of $\langle\bar{Q}_i, \bar{K}_j\rangle$ only depend on $Q_i,K_j$, and the relative position $m_i-m_j$. Proof completed.
\end{proof}

Combing Corollary \ref{coro:RoPE_valid} with Equations \eqref{eq:RoPE_attention} and \eqref{eq:RoPE_weighted_sum}, we easily derive that RoPE satisfies the general RPE characteristic defined by Equation \eqref{eq:general_RPE_functions} in Definition \ref{def:general_RPE_functions}.

Next, we focus on the space complexity of RoPE input:

\begin{proposition}
\label{prop:rope_space}
The space complexity for RoPE (by Definition \ref{def:RoPE}) inputs  in parallel computation is \(\mathcal{O}(NH(2d_k+d_v))\).
\end{proposition}

\begin{proof}
We observe that the intermediate variables in RoPE are $\hat{Q}_i,\hat{K}_j\in \mathbb{R}^{d_k}$.
There are $N$ individual terms for each of these types of variables given that $i,j\in \{1,2,\cdots, N\}$, which totally requires $O(HNd_k)$ space for storage. Combining with the space \(\mathcal{O}(NH(2d_k+d_v))\) required for the original QKV vectors, as indicated by Proposition \ref{prop:mha_space}, RoPE still requires \(\mathcal{O}(NH(2d_k+d_v))\) space for parallel computation. Proof completed.
\end{proof}

RoPE achieves such an efficient space complexity because it encodes relative position by embedding the \textit{absolute} position of tokens into their QK vectors using rotary matrices, as shown in Equation \eqref{eq:absolute_pos_emb}.
By this way, relative positions are naturally encoded into the attention weights after conducting dot product by the promise of rotary matrix properties.
Compared to RPE that individually calculates the individual relative positions between all pairs of tokens, which introduces $O(N^2)$ intermediate variables of QKV vectors, this absolute position-based manner only introduces $O(N)$ such variables that maintain the same complexity of multi-head attention.
Thus, RoPE accesses both compact space and proper encoding of relative positions.

\subsection{Infeasibility of RoPE to handle periodicity}
\label{sec:rope_limitation}

In this section, we illustrate the limitation of RoPE in dealing with angles $\theta$ caused by its absolute position-based encoding.

Compared to RPE that can simply address the periodicity of angles by normalizing the relative angle by $\theta_i-\theta_j \bmod 2\pi$, RoPE \textbf{cannot operate directly} on the relative position (angle) because it embeds this feature implicitly.
Only the absolute angle is visible to RoPE, while the relative one is invisibly encoded by conducting attention between absolute angle-embedded tokens.
Here is an example of this infeasibility.

\begin{example}
Consider three elements \( E_0, E_1, E_2 \) with absolute angles \( \theta_0 = \frac{\pi}{2} \), \( \theta_1 =  0\), and \( \theta_2 = \frac{3\pi}{2} \). We have that
\begin{equation}
\theta_0 - \theta_1 \bmod 2\pi = \theta_1-\theta_2\bmod 2\pi = \frac{\pi}{2},
\end{equation}
which means that the two pairs have identical relative angles.
Consider RoPE that outputs:
\begin{gather}
    \langle \hat{Q}_0, \hat{K}_1 \rangle =f\left(Q_0,K_1, \left\{R\left(\frac{\pi}{2}\theta_l\right)\right\}_{l=0}^{d_k-1}\right),\\
     \langle \hat{Q}_1, \hat{K}_2 \rangle =f\left(Q_1,K_2, \left\{R\left(-\frac{3\pi}{2}\theta_l\right)\right\}_{l=0}^{d_k-1}\right),
\end{gather}
where $R(\theta)$ is the rotary matrix, and $\theta_l = 10000^{-l/d_k}$.
Hence, we easily derive that $\frac{\pi}{2}\theta_l \bmod 2\pi \neq -\frac{3\pi}{2}\theta_l \bmod 2\pi$ for $l\neq 0$.
This result derives:
\begin{equation}
  \left\{R\left(\frac{\pi}{2}\theta_l\right)\right\}_{l=0}^{d_k-1}  \neq \left\{R\left(-\frac{3\pi}{2}\theta_l\right)\right\}_{l=0}^{d_k-1},
\end{equation}
by the property of rotary matrices.
Combining with the formulation\footnote{See the proof of Corollary \ref{coro:RoPE_valid} for details.} of $f(\cdot)$, this result judges the in-equivalence of these two identical relative positions in RoPE.
\end{example}

Therefore, RoPE is not yet feasible for the scenario where angles presents as part of positions, requiring further refinement.

\subsection{Directional rotary position embedding}
\label{sec:rope_spatial}

In this section, we introduce the Directional Rotary Position Embedding (DRoPE) to accurately embed angular information for agent interaction modeling.

Consider the original global position embedding function in RoPE, defined as \( f^\rightarrow\langle R,\theta_l\rangle \), which employs a consecutive 2D rotary transformation with varying scalar values $\{\theta_l\}_{l=1}^{d_k}$. These differing scalar values disrupt the inherent periodicity of the rotary transformation with respect to relative angles. To recover this essential periodic property, DRoPE unifies the scalar value across dimensions, resulting in the following simplified global angle embedding function:
\begin{equation}
\label{eq:absolute_angle_emb}
f^\angle(X,\theta) = \operatorname{BlockDiag}\big(R(\theta), \dots, R(\theta)\big)X,
\end{equation}
where \(\theta\) represents the global angle corresponding to an agent’s heading, \(X \in \mathbb{R}^{2d_k}\) denotes the query-key (QK) vectors of agents, and \(R(\cdot)\) is the standard 2D rotary transformation.

By setting all scalar parameters \(\theta_l\) to unity, Equation \eqref{eq:absolute_angle_emb} reinstates the periodic nature of rotary transformations concerning relative angular differences. This property is formalized in the following proposition:

\begin{proposition}
\label{prop:DRoPE_valid}
For tokens \(i,j\) with QK vectors \(Q_i,K_j\in \mathbb{R}^{2d_k}\), let their global heading angles be denoted as \(\theta_i,\theta_j\), respectively. Define \(\bar{Q}_i\) and \(\bar{K}_j\) as:
\begin{equation}
    \bar{Q}_i = f^\angle(Q_i,\theta_i),\,\, \bar{K}_j = f^\angle(K_j,\theta_j),
\end{equation}
where \(f^\angle(\cdot)\) is the global angle embedding function defined in Equation \eqref{eq:absolute_angle_emb}. Then, their dot product satisfies:
\begin{equation}
    \langle\bar{Q}_i, \bar{K}_j\rangle = f(Q_i,K_j,\theta_i-\theta_j \bmod 2\pi),
\end{equation}
which depends solely on \(Q_i, K_j\), and the periodic relative angle.
\end{proposition}

\begin{proof}
    The dot product \(\langle\bar{Q}_i, \bar{K}_j\rangle\) can be expanded as follows:
    \begin{equation*}
    \begin{aligned}
        \langle\bar{Q}_i, \bar{K}_j\rangle 
        &= Q_i^T  \operatorname{BlockDiag}\big(R(\theta_i)^T,\dots,R(\theta_i)^T\big)\\         &\quad\operatorname{BlockDiag}\big(R(\theta_j),\dots,R(\theta_j)\big) K_j\\
        &= Q_i^T  \operatorname{BlockDiag}\big(R(-\theta_i)R(\theta_j),\dots,R(-\theta_i)R(\theta_j)\big) K_j\\
        &= Q_i^T  \operatorname{BlockDiag}\big(R(\theta_j-\theta_i),\dots,R(\theta_j-\theta_i)\big) K_j,
    \end{aligned}
    \end{equation*}
    leveraging the property of 2D rotary transformations. Given the periodicity property \(R(\theta) = R(\theta \bmod 2\pi)\), we further have:
    \begin{equation*}
        \langle\bar{Q}_i, \bar{K}_j\rangle = Q_i^T  \operatorname{BlockDiag}\big(\{R(\theta_j-\theta_i \bmod 2\pi)\}\big) K_j,
    \end{equation*}
    which explicitly depends only on \(Q_i, K_j\), and the periodic relative angle \(\theta_i-\theta_j \bmod 2\pi\). This completes the proof.
\end{proof}

Proposition \ref{prop:DRoPE_valid} formally demonstrates that DRoPE correctly encodes relative angular information into the embedding space. Consequently, combining DRoPE with RoPE enables simultaneous embedding of agents' relative spatial positions and moving directionalities in an efficeint way. Detailed descriptions of implementable model architectures incorporating DRoPE and RoPE are provided in the subsequent section.

\begin{figure}[!t]
    \centering
    \includegraphics[width=1.\linewidth]{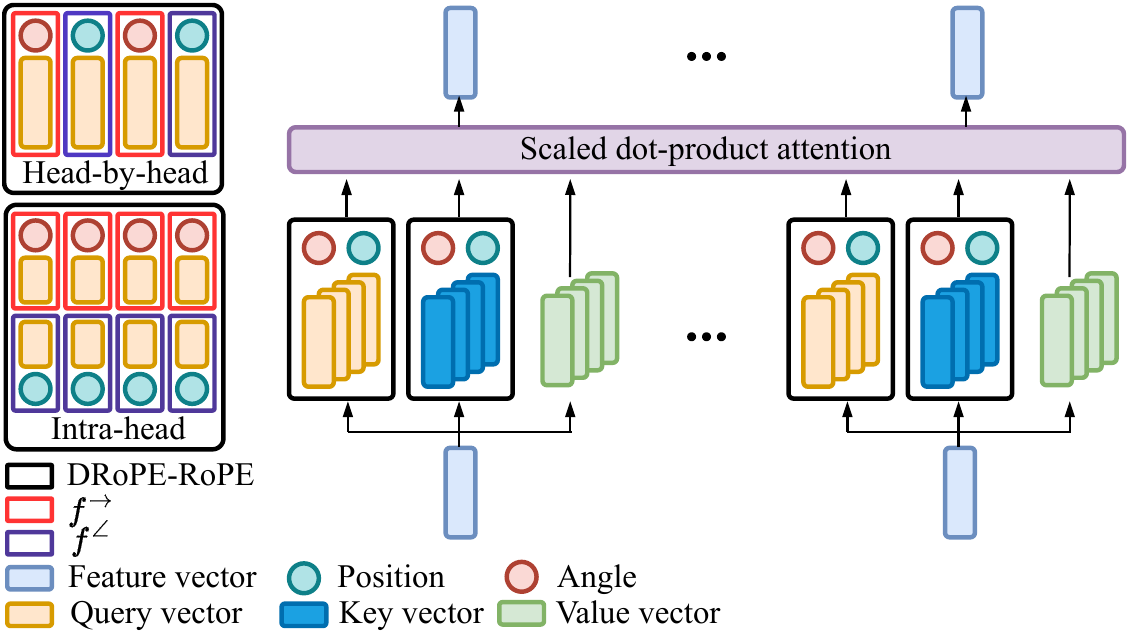}
    \caption{Comparison of two integration methods for DRoPE and RoPE.}
    \label{fig:integration}
\end{figure}

\section{Trajectory Generation with DRoPE-RoPE}
This section presents the practical integration of DRoPE and RoPE into trajectory generation models. We first describe how RoPE and DRoPE are jointly incorporated into multi-head attention modules to accurately embed both relative positions and angles. Subsequently, we outline the overall model architecture built around this combined DRoPE-RoPE attention module.
\subsection{DRoPE-RoPE multi-head attention}
\label{sec:drope_rope}

We introduce two alternative manners to implement DRoPE-RoPE attention module for relative position and angle embedding, named head-by-head and intra-head integration. These two approaches are illustrated in Fig.~\ref{fig:integration}.

Here, we denote the embedding vectors of agents as \(\mathcal{A}=\{A_i\}_{i=1}^N\), with corresponding query, key, and value (QKV) vectors represented as \(\{Q_{\mathcal{A},i}, K_{\mathcal{A},i}, V_{\mathcal{A},i}\}_{i=1}^{N}\). Additionally, let \(\mathcal{P}_\mathcal{A}=\{\text{pos}_{\mathcal{A},i}\}_{i=1}^N\) and \(\mathcal{D}_\mathcal{A}=\{\theta_{\mathcal{A},i}\}_{i=1}^N\) denote the global coordinate positions and heading angles of agents, respectively.
To provide a more concrete explanation of cross-attention, we introduce an additional set of input tokens, denoted as \(\mathcal{M}\), which can represent elements such as scene map line tokens. Similarly, we define their corresponding global coordinate positions and angular information as \(\mathcal{P}_\mathcal{M}\) and \(\mathcal{D}_\mathcal{M}\), respectively.

\noindent\textit{Head-by-head integration.} \quad 
A natural approach is to apply the positional transformation \( f^{\rightarrow} \) and the angular transformation \( f^\angle \) separately to the QK vectors of different heads. The attention scores are computed as follows:

\begin{gather}
\alpha_{ij}^h =
\begin{cases}
\mathrm{softmax} \left(
  \frac{ \langle \hat{Q}_{\mathcal{A},i}^h, \hat{K}_{\mathcal{A},j}^h \rangle}{\sqrt{d_k}}
\right), & \text{if } h \bmod 2 = 0 \\[10pt]
\mathrm{softmax} \left(
  \frac{ \langle \bar{Q}_{\mathcal{A},i}^{h}, \bar{K}_{\mathcal{A},j}^{h} \rangle}{\sqrt{d_k}}
\right), & \text{otherwise}
\end{cases} \\
\text{where }
\hat{Q}^h_{\mathcal{A},i} = f^{\rightarrow}(Q_{\mathcal{A},i}^h, \text{pos}_{\mathcal{A},i}), \quad \hat{K}^{h}_{\mathcal{A},j} = f^{\rightarrow}(K_{\mathcal{A},j}^h, \text{pos}_{\mathcal{A},j}), \\[5pt]
\bar{Q}^{h}_{\mathcal{A},i} = f^\angle(Q_{\mathcal{A},i}^h, \theta_{\mathcal{A},i}), \quad \bar{K}^{h}_{\mathcal{A},j} = f^\angle(K_{\mathcal{A},j}^h, \theta_{\mathcal{A},j}).
\end{gather}

By separately applying RoPE and DRoPE to the query and key vectors of different attention heads, the relative positional and angular relationships between agent $i$ and other agents are implicitly encoded in the attention scores. These relationships subsequently influence the output vectors $\{O_i^h\}_{h=1}^{H}$ as computed via Eq.~\eqref{eq:absolute_pos_emb}. After combining these outputs across attention heads using Eq.~\eqref{eq:mult_attention_outcome} and passing the result through a feed-forward network (FFN), the final representation \( O_i \) effectively encodes both relative positional and angular relationships among agents. We denote this entire operation as $\text{MHSA}_\text{S}^\text{HbH}(\mathcal{A}, \mathcal{P}_\mathcal{A}, \mathcal{D}_\mathcal{A})$. 
Similarly, if we replace the key and value (KV) vectors in the above process with those corresponding to \(\mathcal{M}\), we can define an analogous operation, denoted as 
\( \text{MHCA}_\text{S}^\text{HbH}(\mathcal{A}, \mathcal{M}, \mathcal{P}_\mathcal{A}, \mathcal{P}_\mathcal{M}, \mathcal{D}_\mathcal{A}, \mathcal{D}_\mathcal{M}) \).

\noindent\textit{Intra-head integration.} \quad 
An alternative approach is to decompose the QK vectors into two sub-vectors:
\begin{equation}
Q_{\mathcal{A},i}^h = \big[ Q_{\mathcal{A},i}^{h,\text{pos}}, Q_{\mathcal{A},i}^{h,\text{angle}} \big], \quad
K_{\mathcal{A},j}^h = \big[ K_{\mathcal{A},j}^{h,\text{pos}}, K_{\mathcal{A},j}^{h,\text{angle}} \big],
\end{equation}
where \( Q_{\mathcal{A},i}^{h,\text{pos}}, K_{\mathcal{A},j}^{h,\text{pos}} \in \mathbb{R}^{d_{\text{pos}}} \) , \( Q_{\mathcal{A},i}^{h,\text{angle}}, K_{\mathcal{A},j}^{h,\text{angle}} \in \mathbb{R}^{d_{\text{angle}}} \) and \( Q_{\mathcal{A},i}^{h}, K_{\mathcal{A},j}^{h} \in \mathbb{R}^{2d_k} \),satisfying:
\begin{equation}
d_{\text{pos}} + d_{\text{angle}} = 2d_k.
\end{equation}
Each sub-vector is then transformed separately using \( f^{\rightarrow} \) and \( f^\angle \), leveraging the additivity property of inner products. The attention score is computed as follows:
\begin{equation}
\alpha_{ij}^h = \mathrm{softmax} \left(
\frac{ \langle \hat{Q}_{\mathcal{A},i}^{h,\text{pos}}, \hat{K}_{\mathcal{A},j}^{h,\text{pos}} \rangle 
+ \langle \bar{Q}_{\mathcal{A},i}^{h,\text{angle}}, \bar{K}_{\mathcal{A},j}^{h,\text{angle}} \rangle }{\sqrt{d_k}}
\right),
\end{equation}

where
\begin{equation}
\hat{Q}_{\mathcal{A},i}^{h,\text{pos}} = f^{\rightarrow}(Q_{\mathcal{A},i}^{h,\text{pos}}, \text{pos}_{\mathcal{A},i}), \quad 
\hat{K}_{\mathcal{A},j}^{h,\text{pos}} = f^{\rightarrow}(K_{\mathcal{A},j}^{h,\text{pos}}, \text{pos}_{\mathcal{A},j}),
\end{equation}
\begin{equation}
\bar{Q}_{\mathcal{A},i}^{h,\text{angle}} = f^\angle(Q_{\mathcal{A},i}^{h,\text{angle}}, \theta_{\mathcal{A},i}), \quad
\bar{K}_{\mathcal{A},j}^{h,\text{angle}} = f^\angle(K_{\mathcal{A},j}^{h,\text{angle}}, \theta_{\mathcal{A},j}).
\end{equation}

This integration method ensures that relative positional relationships are effectively embedded within the computation of attention scores. We denote this operation as \( \text{MHSA}_\text{S}^\text{IH}(\mathcal{A}, \mathcal{P}_\mathcal{A}, \mathcal{D}_\mathcal{A}) \). 
Similarly, we define the cross-attention counterpart as  
\( \text{MHCA}_\text{S}^\text{IH}(\mathcal{A}, \mathcal{M}, \mathcal{P}_\mathcal{A}, \mathcal{P}_\mathcal{M}, \mathcal{D}_\mathcal{A}, \mathcal{D}_\mathcal{M}) \).

\subsection{Problem definition}

In this section, we introduce the trajectory generation task, to which we apply DRoPE and RoPE. Zhao et al.~\cite{10803038} redefined the trajectory generation problem using a kinematic model as follows:
\begin{equation}
\begin{aligned}
\label{eq:task_reformal}
& \underset{\theta}{\arg \max }\, \prod_{t=0}^{T-1}P_{\theta}\left(U^\text{a}_t \mid \mathcal{S}^\text{w}_{\leq t}, \mathcal{S}^\text{a}_{\leq t}\right) \\
& \text{subject to} \quad S^\text{a}_{\tau +1} = \mathcal{K}(S_\tau^\text{a}, U_\tau^\text{a}) \\
& t \in \{0,1,\dots, T-1\},\, \tau\in \{0,1,\dots,t\}
\end{aligned}
\end{equation}
where \( S^a_{\leq t} \) represents the historical states of the target agent, including position, yaw, velocity, etc. \( U^a_t \) denotes the control actions of agent \( a \), which, following the definition in KiGRAS \cite{10803038}, consist of acceleration and yaw rate. \( S^w \) represents the world state, encompassing other agents and environmental information such as the map. \( T \) denotes the time window length, which is set to 8s in our case. The function \( \mathcal{K} \) defines the kinematic model that propagates the state of agent \( a \) from \( S_\tau^\text{a} \) at time step \( \tau \) to \( S^\text{a}_{\tau +1} \) based on the control action \( U_\tau^\text{a} \).

\subsection{Model architecture}

In this section, we introduce the architecture of our model. For ease of description, we do not distinguish between head-by-head and intra-head integration in this section. Instead, we use \( \text{MHSA}_\text{S} \) and \( \text{MHCA}_\text{S} \) as unified notations. A detailed analysis of the performance of these two integration methods is presented in Sec.~\ref{exp:different_rpe}.

We encode the static attributes and velocity information of $n$ agents over a time sequence $T$ into agent tokens, denoted as $\mathcal{A} = \{A_i^t\}_{i=1,t=1}^{n,T}$. Similarly, we encode $m$ map segments into map tokens using a subgraph-based encoding, denoted as $\mathcal{M} = \{M_i\}_{i=1}^{m}$. Notably, these tokens do not contain explicit positional or angular information, as spatial interactions are incorporated later.
First, at each time step $t$, we model interactions among agent tokens. Let $\mathcal{P}_{A^t}$ represent the global coordinates of agent tokens at time $t$, i.e., $\mathcal{A}^t = \{A_i^t\}_{i=1}^{n}$, and let $\mathcal{D}_{A^t}$ denote their global orientation angles. The agent-agent interaction can be formalized as:
\begin{gather}
    \text{MHSA}_\text{S}(\mathcal{A}^t, \mathcal{P}_{A^t}, \mathcal{D}_{A^t})
\end{gather}
Next, we model interactions among the $m$ map tokens. Let $\mathcal{P}_M$ represent the global coordinates of the map tokens $\mathcal{M} = \{M_i\}_{i=1}^{m}$, and let $\mathcal{D}_M$ denote their global orientation angles:
\begin{gather}
    \text{MHSA}_\text{S}(\mathcal{M}, \mathcal{P}_M, \mathcal{D}_{M})
\end{gather}
Subsequently, at each time step $t$, we model interactions between agent tokens and map tokens:
\begin{gather}
    \text{MHCA}_\text{S}(\mathcal{A}^t, \mathcal{M}, \mathcal{P}_{A^t}, \mathcal{P}_M, \mathcal{D}_{A^t}, \mathcal{D}_{M})
\end{gather}
Finally, we incorporate temporal positional encoding (PE) into all agent tokens. Subsequently, we model the temporal interactions within each agent \( i \)'s token sequence, denoted as \( \mathcal{A}_i = \{\mathcal{A}_i^t\}_{t=1}^{T} \), using a standard causal self-attention transformer layer.

After completing all interaction steps, we decode the final agent tokens via an MLP to obtain the probability distribution of each agent's control actions at each time step. The predicted distributions are then trained with the ground-truth labels using the cross-entropy loss function.

\section{Experiment}

\subsection{Dataset and metrics}
We conducted experiments using version 1.2 of the Waymo Motion Dataset \cite{sun2020scalability}. We performed 8-second closed-loop simulations with our model and submitted the results to the Waymo SimAgent Challenge for a fair comparison with other methods. Specifically, we employed minADE to assess the accuracy of trajectory predictions, REALISM to evaluate the authenticity of the generated trajectories, and the total number of model parameters to quantify model size.

In the ablation study, we utilized kinematic metrics, interactive metrics, and map-based metrics to further assess the results from three perspectives: the kinematic realism of agent motion, the realism of interactions among agents, and the realism of interactions between agents and the map. For local evaluation on the validation dataset, we utilized the tools provided by Zhang et al. \cite{zhang2025closed}.

Additionally, we measured the peak memory usage during model training and evaluation, as well as the FLOPs required during evaluation, to comprehensively compare the computational cost of different scene representation paradigms.

\begin{table}[h]
\footnotesize
\setlength{\tabcolsep}{3pt}
\caption{Comparison results of ours and state-of-the-art approaches in SimAgents Challenge.}
\centering
\begin{tabular}{@{}lrccc@{}}
\toprule
Method           & Params & SR       & minADE & REALISM \\ \midrule
SMART-tiny-CLSFT & 7M     & query-centric   & 1.3068 & \textbf{0.7702}  \\
UniMM            & 4M     & query-centric   & 1.2947 & 0.7684  \\
SMART-large      & 101M   & query-centric   & 1.3728 & 0.7614  \\
KiGRAS           & 0.7M   & agent-centric   & 1.4384 & 0.7597  \\
SMART-tiny       & 7M     & query-centric   & 1.4062 & 0.7591  \\
BehaviorGPT      & 3M     & query-centric   & 1.4147 & 0.7473  \\
GUMP             & 523M   & scene-centric   & 1.6041 & 0.7431  \\
MVTE             & 65M    & query-centric   & 1.6770 & 0.7302  \\
VBD              & 12M    & query-centric   & 1.4743 & 0.7200  \\
TrafficBOTv1.5   & 10M    & scene-centric   & 1.8825 & 0.6988  \\
DRoPE-Traj             & 3M     & query-centric   & \textbf{1.2626} & 0.7625  \\ \bottomrule
\end{tabular}
\label{tab:cmp}
\end{table}

\subsection{Implementation details}
We process data at a frequency of 2 Hz. The map polylines are divided into segments with a maximum length of 25 meters. For each segment, we use the midpoint as its representative position, and the angle between the midpoint and the next point as its heading. For stop signs in the scene, which consist of only a single point, we set their heading to 0. 

To normalize the polylines, we transform each one into its own local coordinate system, preserving its shape information while removing absolute positional data.
All agent tokens and map tokens are embedded into 64-dimensional feature vectors.

We train the model using NVIDIA H20 with a batch size of 64 and a learning rate of $1 \times 10^{-2}$. The Adan \cite{xie2024adan} optimizer is employed to optimize the model.

\subsection{Performance comparison}
\label{sec:performance_cmp}

We conducted a rigorous and fair comparison with state-of-the-art methods from the Waymo SimAgent Challenge leaderboard, including UniMM \cite{lin2025revisit}, SMART \cite{wu2024smart}, SMART-CLSFT \cite{zhang2025closed}, BehaviorGPT \cite{zhou2024behaviorgpt}, GUMP \cite{hu2024solving}, MVTE \cite{wang2023multiverse}, VBD \cite{huang2024versatile}, and TrafficBOTv1.5 \cite{zhang2024trafficbots}. Our proposed method is referred to as DRoPE-Traj.

The evaluation results are presented in Table \ref{tab:cmp}, where we also provide the scene representation (SR) method used by each approach. Unlike other query-centric methods, our approach maintains the same spatiotemporal complexity as scene-centric approaches while achieving the lowest spatiotemporal complexity overall. Additionally, our model\footnote{We employ the head-by-head integration approach.} achieves SOTA minADE performance (1.2626) while maintaining a REALISM score (0.7625) comparable to SMART-tiny-CLSFT (0.7702).

\subsection{Evaluation of space and computational efficiency}
In this section, we compare our DRoPE-RoPE-based model with the scene-centric approach and the RPE-based query-centric approach in terms of efficiency. The comparison is conducted in terms of training memory consumption, evaluation memory consumption, and FLOPs during evaluation. Specifically, we replace the transformer layers in our network with (1) standard transformer layers, representing the scene-centric approach, (2) transformer layers incorporating RPE, which follow the conventional query-centric representation, and (3) two transformer variants integrating RoPE and DRoPE, as introduced in Sec.~\ref{sec:drope_rope}. Since these two variants have identical computational and memory costs, we refer to them collectively as DRoPE-RoPE. To analyze the impact of increasing parameter size on computational costs, we vary the embedding dimension of the QKV vectors. All measurements are conducted with a fixed batch size of 1. Results are shown in Fig.~\ref{fig:computation_res}.
\begin{figure}[!t]
    \centering
    \includegraphics[width=1.\linewidth]{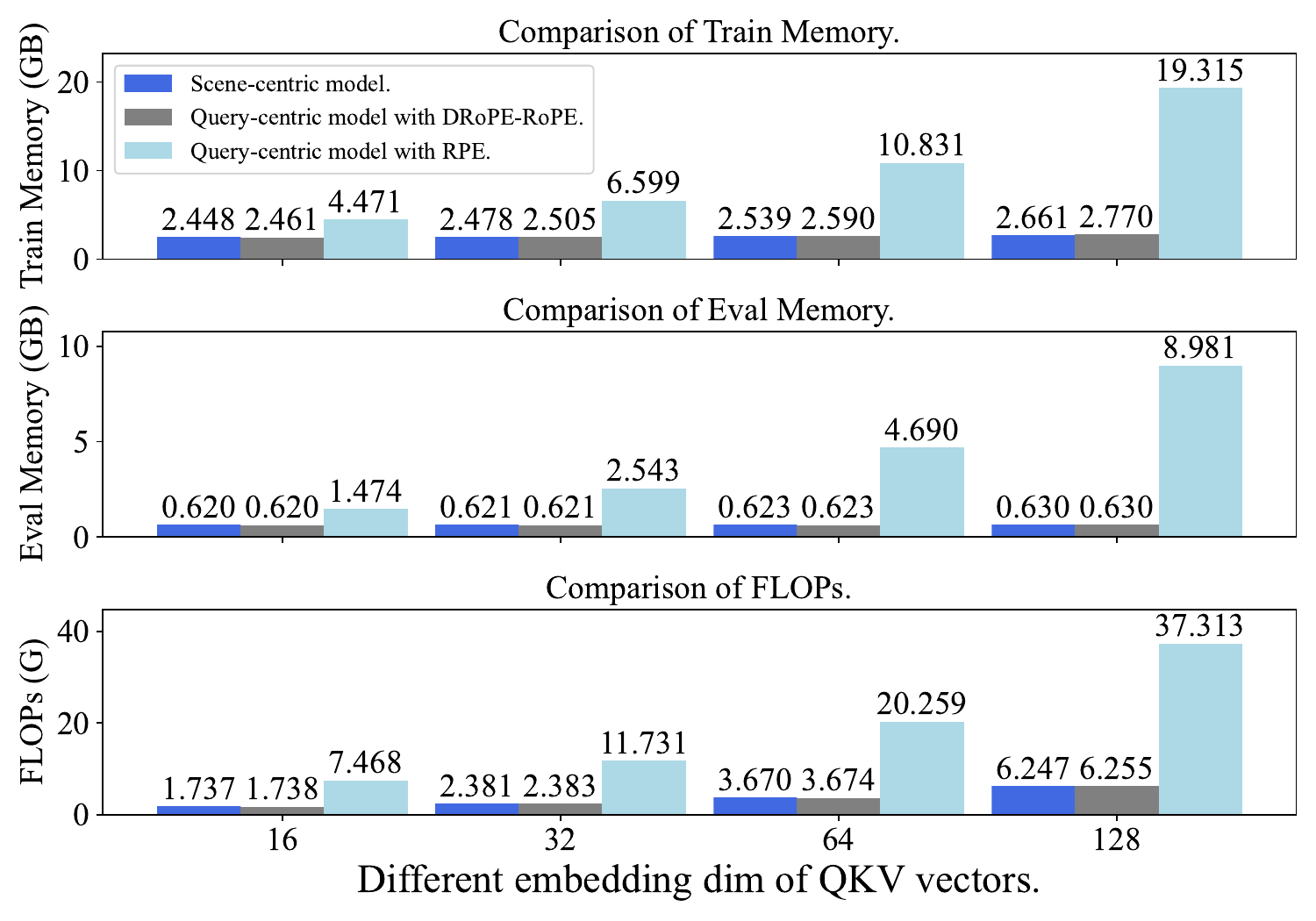}
    \caption{Comparison of training memory, evaluation memory, and FLOPs across different scene representation approaches.}
    \label{fig:computation_res}
\end{figure}
It can be observed that, in terms of both memory consumption and FLOPs, our query-centric approach with DRoPE-RoPE is nearly identical to the scene-centric approach. In contrast, as the embedding dimension of QKV increases, the memory consumption of RPE exhibits an exponential-like surge. Due to this issue, almost all RPE-based methods restrict each scene element to interact with only a limited number of nearby elements to mitigate the overwhelming memory cost.  

Regarding FLOPs, although both approaches theoretically have a time complexity of $O(N^2)$, as shown in Eq.~\ref{eq:rpe_k_embedding} and Eq.~\ref{eq:rpe_v_embedding}, RPE introduces an additional $O(N^2)$ operation by encoding pairwise relative positions into the KV vectors using an MLP. This results in a significant computational overhead, leading to a 4-6× increase in FLOPs for RPE compared to DRoPE-RoPE across different embedding dimensions.

\subsection{Ablation on different DRoPE-RoPE architectures}
\label{exp:different_rpe}
In this subsection, we evaluate the performance of intra-head and head-by-head integration, as well as RPE, on the validation set. To ensure a fair comparison, we use our backbone network while replacing only the transformer layers. The batch size is fixed at 64. We represent the scene using 1024 map tokens. Additionally, to enable stable training for RPE-based methods, each element is restricted to attending to its 50 nearest neighbors. The results are presented in Table~\ref{tab:ablation_study}.
\begin{table}[h]
\footnotesize
\setlength{\tabcolsep}{3pt}
\caption{Comparison results of different DRope-RoPE and RPE.}
\centering
\begin{tabular}{@{}lcccc@{}}
\toprule
Method       & minADE $\downarrow$ & \makecell[c]{Kinematic \\ metrics $\uparrow$} & \makecell[c]{Interactive \\ metrics $\uparrow$} & \makecell[c]{Map-based \\ metrics $\uparrow$} \\ \midrule
RPE          & 1.3910      & 0.4820      & 0.7878      & 0.8416 \\
Intra-head   & 1.4289 & 0.4804 & 0.7843 & 0.8345 \\
Head-by-head & \textbf{1.3745} & \textbf{0.4827} & \textbf{0.7894} & \textbf{0.8449} \\ \bottomrule
\end{tabular}
\label{tab:ablation_study}
\end{table}
Despite both DRoPE and RoPE being theoretically feasible for aggregation, the intra-head integration approach shows a noticeable performance drop compared to head-by-head integration. Specifically, the Kinematic metrics score decreases from 0.4827 to 0.4804, while both Interactive metrics and Map-based metrics also exhibit declines. Notably, minADE increases from 1.3745 to 1.4289. We hypothesize that this degradation occurs because intra-head integration mixes direction and positional features, which have inherently different characteristics, making it more challenging for the network to learn effectively, thereby leading to inferior performance compared to head-by-head integration.  

As for RPE, due to its excessive memory consumption, it can only attend to a limited number of nearby elements, resulting in suboptimal performance compared to head-by-head integration, which can attend to all elements. This constraint leads to an increase of 0.02m in minADE compared to head-by-head integration, along with slight declines in other metrics.

\section{Conclusion}
In this paper, we introduce Directional Rotary Position Embedding (DRoPE), a novel extension of Rotary Position Embedding (RoPE) designed to efficiently model periodic angular relations in agent trajectory generation tasks. By incorporating a uniform identity scalar into RoPE's 2D rotary transformation, DRoPE overcomes RoPE’s inherent limitations in handling angular information, making it feasible of integrating DRoPE and RoPE to encode both relative positions and headings in autonomous driving systems. Our theoretical analysis demonstrates that DRoPE retains the space and computational efficiency of scene-centric methods while effectively leveraging relative positional information. Thus, DRoPE simultaneously achieves high accuracy and optimal efficiency in both computational and space complexity.

In future work, we aim to explore further optimizations to DRoPE for more complex interaction scenarios \cite{ren2021deep, tampuu2020survey}, as well as extend its application to other domains requiring efficient periodic angle modeling. Our findings contribute a novel and effective solution to the ongoing challenge of balancing accuracy, time complexity, and space complexity in trajectory generation for autonomous driving.

\bibliographystyle{IEEEtranS} % use IEEEtran.bst style
{\footnotesize % or \footnotesize or \scriptsize
\bibliography{IEEEabrv}
}

\newpage

\end{document}